\documentclass[10pt,letterpaper]{article}
\usepackage{thunder}

\usepackage[utf8]{inputenc} 
\usepackage[T1]{fontenc}    
\usepackage{hyperref}       
\usepackage[nameinlink]{cleveref}  
\usepackage{url}            
\usepackage{booktabs}       
\usepackage{amsfonts}       
\usepackage{nicefrac}       
\usepackage{microtype}      
\usepackage{xcolor}         
\usepackage[numbers,sort]{natbib}
\usepackage{amsthm}
\usepackage{amssymb}
\usepackage{booktabs}
\usepackage{multirow}
\usepackage{float}
\usepackage{makecell}
\newtheorem{theorem}{Theorem}

\newtheorem{lemma}{Lemma}
\newtheorem{definition}{Definition}
\newtheorem{assumption}{Assumption}
\newtheorem{remark}{Remark}

\usepackage{bm}
\usepackage{subfigure}
\usepackage{colortbl}

\usepackage{footmisc}  
\makeatletter
\renewcommand{\footnoterule}{%
  \kern -3pt 
  \hrule width \textwidth height 0.4pt
  \kern 2pt  
}
\makeatother
\makeatletter
\newcommand\blfootnote[1]{%
  \begingroup
  \renewcommand\thefootnote{}\footnote{#1}%
  \addtocounter{footnote}{-1}%
  \endgroup
}
\makeatother

\usepackage{lineno}

\definecolor{darkblue}{rgb}{0, 0, 0.5}
\definecolor{darkgreen}{RGB}{50,100,0}
\definecolor{darkred}{RGB}{200, 0, 0}
\definecolor{lightblue}{RGB}{220,235,250}
\usepackage[tikz]{bclogo}
\usepackage{enumitem}
\newenvironment{itemize*}%
 {\leftmargini=20pt\begin{itemize}%
  \setlength{\itemsep}{3pt}%
  \setlength{\parskip}{0pt}%
  }%
 {\end{itemize}}
\newenvironment{enumerate*}%
 {\begin{enumerate}%
  \setlength{\itemsep}{0pt}%
  \setlength{\parskip}{0pt}}%
 {\end{enumerate}}

\usepackage{caption}
\usepackage{listings}
\lstnewenvironment{normalcode}{
  \lstset{
    language=Python,
    basicstyle=\ttfamily\small,
    backgroundcolor=\color{white},
    breaklines=true,
    showstringspaces=false,
    frame=none,
    xleftmargin=1em,
    aboveskip=1pt,
    belowskip=1pt
  }
}{}

\lstnewenvironment{diffaddcode}{
  \lstset{
    language=Python,
    basicstyle=\ttfamily\small,
    backgroundcolor=\color{green!15},
    breaklines=true,
    showstringspaces=false,
    frame=none,
    xleftmargin=1em,
    aboveskip=1pt,
    belowskip=1pt
  }
}{}

\lstnewenvironment{diffdelcode}{
  \lstset{
    language=Python,
    basicstyle=\ttfamily\small,
    backgroundcolor=\color{red!15},
    breaklines=true,
    showstringspaces=false,
    frame=none,
    xleftmargin=1em,
    aboveskip=1pt,
    belowskip=1pt
  }
}{}
\DeclareCaptionType{listing}[Listing][List of Listings]
\usepackage[most,skins,theorems]{tcolorbox}

\newcolumntype{P}[1]{>{\centering\arraybackslash}p{#1}}

\title{Revisiting LLM Reasoning via Information Bottleneck}

\author[1,\footnotesize{\textbf{*}}]{Shiye Lei}
\author[2,\footnotesize{$\dagger$}]{Zhihao Cheng}
\author[2]{Kai Jia}
\author[3,\footnotesize{$\dagger$}]{Dacheng Tao}
\affil[1]{The University of Sydney}
\affil[2]{ByteDance}
\affil[3]{Nanyang Technological University}

\date{}

\begin{document}
\maketitle
\blfootnote{\footnotesize{\textbf{*}}Work done during an internship at ByteDance.}
\blfootnote{\footnotesize{$\dagger$} Corresponding authors: \href{mailto:zhihao.cheng@bytedance.com}{zhihao.cheng@bytedance.com},\ \href{mailto:dacheng.tao@ntu.edu.sg}{dacheng.tao@ntu.edu.sg}}

\begin{abstract}
Large language models (LLMs) have recently demonstrated remarkable progress in reasoning capabilities through reinforcement learning with verifiable rewards (RLVR). By leveraging simple rule-based rewards, RL effectively incentivizes LLMs to produce extended chain-of-thought (CoT) reasoning trajectories, progressively guiding them toward correct answers. However, existing approaches remain largely heuristic and intuition-driven, limiting the development of principled methodologies. In this paper, we present a theoretical characterization of LLM reasoning grounded in information bottleneck (IB) principle, introducing \textit{IB-aware reasoning optimization} (IBRO), a framework that encourages reasoning trajectories to be both \textit{informative} about the final correct answer and \textit{generalizable} across diverse prompts. We derive a practical token-level surrogate objective and propose an efficient approximation, resulting in the lightweight \textit{IB regularization} method. This technique integrates seamlessly into existing RL-based post-training frameworks without additional computational overhead, requiring only a one-line code modification. Empirically, we validate IB regularization across multiple mathematical reasoning benchmarks and RL algorithms, demonstrating consistent improvements in LLM reasoning performance.
\end{abstract}

\section{Introduction}
Benefiting from the extensive knowledge embedded in pre-trained large language models (LLMs), post-training has increasingly focused on reinforcement learning (RL) as a means to further enhance human alignment \citep{ouyang2022training} and reasoning capabilities \citep{jaech2024openai}. Recent studies have demonstrated that reinforcement learning with verifiable rewards (RLVR) provides a scalable and effective approach for incentivizing reasoning in LLMs \citep{guo2025deepseek}. This approach relies solely on rule-based supervision, without requiring explicit process-level rewards or supervised chain-of-thought (CoT) annotations. Post-training with RLVR encourages LLMs to spontaneously generate coherent reasoning chains, resulting in substantial gains on complex tasks in coding and mathematics.

Although recent empirical advances are promising, approaches for enhancing LLM reasoning remain largely heuristic and intuition-driven, limiting the development of principled methodologies. For instance, exploration, a core characteristic of RL, is critical for discovering high-quality reasoning trajectories. Accordingly, prior works often advocate heuristically maintaining high generation entropy, {\it i.e.}, encouraging token-level uncertainty, during post-training \citep{cui2025entropy,cheng2025reasoning,yao2025diversity}. In contrast, another line of research suggests that explicitly reducing entropy or uncertainty, even in the absence of reward signals, can lead to improved reasoning performance \citep{agarwal2025unreasonable,gao2025one,li2025confidence}. These conflicting findings underscore the need for a rigorous theoretical understanding of reasoning in LLMs, which remains elusive yet essential.

In this work, we address this theoretical gap by analyzing LLM reasoning from the perspective of information bottleneck (IB) principle \citep{tishby2000information, tishby2015deep}, which emphasizes the importance of discarding irrelevant information while preserving task-relevant signals. We introduce \textit{IB-aware reasoning optimization} (IBRO), an information-theoretic framework designed to optimize LLM reasoning capability (\Cref{def:main}). Specifically, IBRO encourages reasoning processes to maximize informativeness with respect to ({\it w.r.t.}) correct answers while minimizing dependency on irrelevant, prompt-specific details. We then derive a token-level surrogate IBRO objective (\Cref{thm:main}) and establish a high-probability generalization bound to theoretically justify the IBRO formulation (\Cref{thm:bound}). To facilitate practical implementation, we derive an efficient approximation of the IBRO objective, resulting in a novel \textit{IB regularization} term. Concretely, IB regularization modulates the token-level entropy based on their corresponding advantages, incentivizing higher entropy for critical tokens and penalizing uninformative ones. Our IB regularization seamlessly integrates into existing RL-based post-training frameworks, introducing negligible computational overhead and requiring only a single line of code modification.

To comprehensively evaluate our IB regularization, we perform post-training using two representative RL algorithms: PPO \citep{schulman2017proximal} and DAPO \citep{yu2025dapo}, which correspond to the mainstream with-critic and without-critic paradigms, respectively \citep{shao2024deepseekmath}. All experiments are conducted on Qwen2.5-7B \citep{yang2024qwen2_5}, a widely adopted LLM base model that has not been specifically optimized for instruction following or reasoning. We evaluate performance on multiple mathematical reasoning benchmarks, including AMC23 and AIME24/25. Across these tasks, we observe consistent and stable improvements, with an average gain of two points on both PPO and DAPO. We further conduct fine-grained analysis {\it w.r.t.} entropy dynamics and response length to demonstrate the stability and compatibility of IB regularization.

\section{Related Works}

\paragraph{LLM Reasoning} Pre-training endows LLMs with vast amounts of knowledge, and a key technique for further enhancing their reasoning capabilities is CoT prompting, which encourages step-by-step problem solving \citep{wei2022chain}. Recent seminal works such as OpenAI-o1 \citep{jaech2024openai} and DeepSeek-R1 \citep{guo2025deepseek} adopt RL post-training to incentivize the emergence of CoT, enabling models to tackle complex reasoning tasks such as mathematics and code generation. Building on these efforts, a growing body of research has sought to understand the key factors in RL post-training that contribute to improved reasoning. One critical factor is generation entropy, which quantifies the uncertainty in the model’s token-level output distribution. In practice, entropy often collapses rapidly toward zero during post-training, leading to overconfident predictions, reduced exploration, and ultimately, diminished reasoning capability. To counter this, \citet{yu2025dapo} propose ClipHigher that relaxes clip constraints to allow more off-policy update for low-probability tokens. In parallel, explicit entropy regularization has gained attention as a direct intervention. However, its efficacy remains debated: on one hand, entropy minimization has been shown to promote reasoning without relying on explicit reward signals \citep{agarwal2025unreasonable, gao2025one, li2025confidence}; on the other hand, several works advocate for maintaining higher entropy to preserve exploration and thus foster reasoning \citep{cui2025entropy,cheng2025reasoning,yao2025diversity}. In this paper, we revisit the challenge of LLM reasoning from an information-theoretic perspective, aiming to balance informativeness and generalization in reasoning process. Our analysis leads to a simple yet effective advantage-aware entropy regularization, which integrates seamlessly into existing RL post-training methods.

\textbf{Information Bottleneck (IB)} posits that an effective latent representation should (1) discard irrelevant information from the input to promote generalization, achieved by minimizing the mutual information between the input and the latent features, (2) while retaining information that is predictive of the target, achieved by maximizing the mutual information between the latent code and the label~\citep{tishby2000information, tishby2015deep}. Although computing mutual information terms is generally intractable, \citet{alemi2017deep} propose a variational lower bound that enables practical estimation. IB principle has received empirical support from~\citet{michael2018on} and theoretical justification from~\citet{kawaguchi2023does}. In this work, we revisit LLM reasoning through the lens of the IB principle and derive a simple yet effective regularization term to enhance reasoning quality. While \citet{yu2025memorization} also analyze LLMs from an IB perspective, their focus is on improving the pre-training phase by minimizing matrix-based entropy~\citep{giraldo2014measures} for compression. In contrast, our work targets the post-training phase and aims to enhance LLM reasoning through a novel IB–based analysis. 

\section{Preliminaries}

\paragraph{RLVR}
Given a question or prompt \(\bm{q}\), a LLM \(\pi_\theta\), parameterized by \(\theta\), generates a response as a sequence of tokens \((o_1, o_2, \ldots, o_T)\) in an autoregressive manner, where each token is sampled according to \(\pi_\theta(o_t \mid o_{<t}, \bm{q})\).  In the RLVR setting, the training dataset is given by \(\mathcal{S} = \{(\bm{q}_i, \bm{a}_i)\}_{i=1}^m\), where \(\bm{a}_i\) is the answer to question \(\bm{q}_i\), without intermediate reasoning chains. Several RL objectives have been developed based on PPO \citep{schulman2017proximal}. For completeness, we briefly recall the standard PPO objective:
\begin{equation*}
    \mathcal{J}_{\texttt{PPO}} = \mathbb{E}_{(\bm{q},\bm{a})\sim \mathcal{S},\,o_{\leq t} \sim \pi_{\theta_{\text{old}}}(\cdot \mid \bm{q})} \left[\min\left(r_t A_t,\, \text{clip}\left(r_t, 1 - \epsilon, 1 + \epsilon \right) A_t\right)\right],
\end{equation*}
where \(r_t = \frac{\pi_\theta(o_t \mid o_{<t}, \bm{q})}{\pi_{\theta_{\text{old}}}(o_t \mid o_{<t}, \bm{q})}\) is the importance sampling ratio, \(A_t = A(o_t; o_{<t}, \bm{q})\) denotes the advantage of selecting token \(o_t\), and the hyperparameter $\epsilon$ controls the clipping range. Intuitively, \(A_t\) quantifies how much better token \(o_t\) is compared to other possible tokens at position \(t\).

While PPO typically requires training a separate critic model to estimate \(A_t\), recent methods such as GRPO eliminates the need for critic learning by introducing a group-normalized reward strategy. Specifically, for each prompt, \(G\) rollouts are sampled, and their corresponding rewards \(\{R_i\}_{i=1}^G\) are used to compute normalized token-level advantages for the \(i\)-th response as $A_{i,t} = \frac{R_i - \text{mean}(R)}{\text{std}(R)}$.

\textbf{Mutual Information (MI)} quantifies the amount of shared information between two random variables:
\begin{equation*}
I(X; Y) = H(X) - H(X \mid Y) = H(Y) - H(Y \mid X),
\end{equation*}
where the entropy \(H(\cdot)\) is defined as \(H(X) = \mathbb{E}_{X\sim p(X)} \left[-p(X) \log p(X) \right]\). For notational simplicity, we slightly abuse notation by letting $X$ and $Y$ denote both random variables and specific realizations. A large \(I(X; Y)\) indicates a strong statistical dependence between \(X\) and \(Y\); that is, observing \(Y\) significantly reduces the uncertainty of \(X\), and vice versa.

\textbf{Information Bottleneck (IB)} provides a principled framework for characterizing the trade-off between a model's \emph{representation complexity} and its \emph{predictive power}. Given a model \(\mathcal{M}\) that first encodes the input \(X\) into a latent representation \(Z\), which is then used to predict the target \(Y\), {\it IB seeks an optimal representation \(Z\)} by solving
\begin{equation*}
    \min_{Z \sim \mathcal{M}(Z \mid X)} \quad I(X; Z) - \beta I(Z; Y),
\end{equation*}
where \( I(X; Z) \) quantifies the amount of information retained about the input, {\it i.e.}, the complexity of the representation, and \( I(Z; Y) \) measures how informative the representation is for predicting the output. The coefficient \( \beta > 0 \) balances compression against predictive accuracy. Minimizing \(I(X; Z)\) penalizes representations that capture unnecessary details from \(X\), thereby encouraging generalization, while maximizing \(I(Z; Y)\) ensures that the representation retains sufficient information for accurate prediction. Together, the IB objective promotes representations that are both compact and task-relevant.

\section{Reasoning via Information Bottleneck}

Given a powerful post-trained LLM $\pi$ and a prompt \(\bm{q}\), the LLM engages in a reasoning process to generate a CoT \(\bm{r}\), with the objective of arriving at the correct answer \(\bm{a}\). Here, $\bm{a}$ denotes the ground truth rather than LLM predictions. This raises a fundamental question: \emph{what characterizes a good reasoning process, or equivalently, a good CoT?} At a high level, a satisfactory CoT should be (1) \textit{informative}, effectively guiding the model to produce the correct final answer; and (2) \textit{generalizable} such that the reasoning process does not depend heavily on the specific prompt and thus transfer well to unseen questions. Motivated by information bottleneck principle, we propose the following formulation to optimize LLM reasoning in an IB-aware manner.
\begin{definition}[IB-Aware Reasoning Optimization]
\label{def:main}
Given a base LLM $\pi$ and a dataset of prompt-answer pairs $(\bm{q},\bm{a})$, we optimize the reasoning ability of $\pi$ by
\begin{equation*}
\label{eq:ib}
    \min_{\pi(\bm{r} \mid \bm{q})} I(\bm{q}; \bm{r}) - \beta I(\bm{r}; \bm{a}),
\end{equation*}
where \(I(\bm{q}; \bm{r})\) quantifies the information retained from the prompt and \(I(\bm{r}; \bm{a})\) measures the informativeness of the reasoning path toward the answer.
\end{definition}
IB-aware reasoning optimization (IBRO) seeks reasoning processes \(\bm{r}\) that minimize dependence on unnecessary details in \(\bm{q}\), while maximizing relevance to the target answer \(\bm{a}\), and the hyperparameter \(\beta > 0\) balances compression and predictiveness.
\begin{remark}
\label{remark:1}
Since overfitting is rarely a concern in the context of LLM RL post-training, response accuracy is typically prioritized over compression. As a result, a large value of  \(\beta > 1\) is preferred to bias the optimization toward maximizing accuracy rather than generalization.
\end{remark}

\subsection{Practical Objective}
While IBRO offers a principled guideline for LLM reasoning optimization, the mutual information terms are intractable and do not naturally align with the token-level training objectives commonly used in LLM fine-tuning. To this end, we derive a more practical objective suitable for both understanding and implementation. First, the mutual information terms can be expressed by entropy as
\begin{equation*}
\label{eq:mi2entropy}
I(\bm{q}; \bm{r}) = H(\bm{r}) - H(\bm{r} \mid \bm{q}), \quad
I(\bm{r}; \bm{a}) = H(\bm{r}) - H(\bm{r} \mid \bm{a}).
\end{equation*}

Since LLM reasoning-oriented post-training primarily aims to improve answer quality in question, the generated reasoning CoT $\bm{r}$ becomes highly conditioned on the input question and is rarely optimized independently. Therefore, we propose a rational assumption as below.
\begin{assumption}
\label{asp:1}
\(\pi(\bm{r})\) remains invariant during LLM RL post-training.
\end{assumption}
Under \Cref{asp:1}, the term \(H(\bm{r}) = \mathbb{E}_{\bm{r} \sim \pi_\theta}[-\log \pi_\theta(\bm{r})]\) can be treated as a constant during post-training. Moreover, leveraging the inequality \(H(\bm{r} \mid \bm{a}) \leq H(\bm{r},\bm{q} \mid \bm{a}) = H(\bm{r} \mid \bm{q}, \bm{a}) + H(\bm{q} \mid \bm{a})\), where $H(\bm{q} \mid \bm{a})$ is a constant that depends only on data distribution, we can obtain the following practical formulation.
\begin{theorem}[Surrogate IBRO objective] 
\label{thm:main}
Assume \Cref{asp:1} holds, and let the reasoning trajectory be \(\bm{r} = (o_1, o_2, \ldots, o_T)\). Then IBRO admits the following upper bound (up to an additive constant):
\begin{equation*}
\label{eq:ib objective}
\min_{\pi(\bm{r}|\bm{q})} \;\; 
\sum_{t=1}^T \left( \beta H\left(o_t \mid o_{<t}, \bm{q}, \bm{a}\right) - H\left(o_t \mid o_{<t}, \bm{q}\right) \right)
\end{equation*}
\end{theorem}
\begin{remark}
(1) \(H(o_t \mid o_{<t}, \bm{q}, \bm{a})\) quantifies how much knowing the correct answer \(\bm{a}\) reduces uncertainty about token \(o_t\). A smaller value indicates that \(o_t\) is highly informative for predicting \(\bm{a}\). Minimizing this term ensures that the reasoning chain \(\bm{r}\) captures key information aligned with the correct answer.
    
(2) \(H(o_t \mid o_{<t}, \bm{q})\) is the standard token-level entropy. Maximizing this term promotes diversity and reduces over-reliance on the prompt \(\bm{q}\), thereby improving reasoning generalization.
\end{remark}

With the IBRO framework, we further derive a generalization bound based on information-theoretic analysis in \citep{kawaguchi2023does}, and the proof detail can be found in \Cref{app:proof_main}.
\begin{theorem}[IBRO generalization bound]
\label{thm:bound}
Let \( \pi \) be a LLM, training dataset $\mathcal{S}=\{(\bm{q}_i,\bm{a}_i)\}_{i=1}^m$ are i.i.d. drawn from the joint data distribution $\mathcal{D}$. 
Suppose the LLM parameters are updated by \( \Delta\theta \) during learning. Let \( \texttt{ACC}(\mathcal{S}) \) and \( \texttt{ACC}(\mathcal{D}) \) denote the empirical and population accuracy of \( \pi \) after training, respectively. Define the IBRO loss as $\mathcal{L}_\texttt{IB}=\beta I(\bm{r} \mid \bm{q},\bm{a}) - I(\bm{r}\mid \bm{q})$ computed {\it w.r.t.} $\mathcal{D}$. If $\beta \geq 2$, then, for any \( \delta > 0 \), with probability at least \( 1 - \delta \) over the sampling of \( \mathcal{S} \), the generalization gap $\Delta(\mathcal{S}) = \vert \texttt{ACC}(\mathcal{S}) - \texttt{ACC}(\mathcal{D}) \vert$ satisfies:
\[
\Delta(\mathcal{S}) \lesssim \sqrt{\frac{\mathcal{L}_\texttt{IB} + \|\Delta\theta\|^2 + \log \frac{1}{\delta}}{m}} + \tilde{\mathcal{O}}\left( \sqrt{ \frac{ \|\Delta\theta\|^2 + 1 }{m} } \right).
\]
\end{theorem}
RL-based post-training typically results in highly sparse parameter updates \citep{mukherjee2025reinforcement} and small KL divergence {\it w.r.t.} initial parameters \citep{rajani2025scalpel}, leading to a small \(\|\Delta\theta\|\). This suggests that the generalization bound is primarily governed by the IBRO loss \(\mathcal{L}_\texttt{IB}\).


\subsection{Information Bottleneck Regularization}
\label{sec:ib reg}
While \Cref{thm:main} provides a practical IBRO objective, computing the conditional entropy term \(H\left(o_t \mid o_{<t}, \bm{q}, \bm{a}\right)\) requires additional rollouts conditioned on the ground truth \(\bm{a}\). However, rollouts are 
time-consuming and often account for more than half of the total training time due to the autoregressive nature of LLMs. This additional requirement introduces significant computational overhead and limits the scalability of the method in practice.

To address this issue, we analyze the token-level surrogate IBRO loss under the setting \(\beta = 2\), as suggested by \Cref{remark:1}, although our analysis applies for other values of \(\beta\) as well. Since the conditional entropy satisfies the bound \(H\left(o_t \mid o_{<t}, \bm{q}, \bm{a}\right) \in [0, H\left(o_t \mid o_{<t}, \bm{q}\right)]\), we have
\[
\ell_\texttt{IB}^t = \beta H\left(o_t \mid o_{<t}, \bm{q}, \bm{a}\right) - H\left(o_t \mid o_{<t}, \bm{q}\right)
\in \left[-H_t,\,  H_t\right],
\]
where \(H_t = H\left(o_t \mid o_{<t}, \bm{q}\right)\) denotes the token-level entropy. The expression above can be rewritten as \(\ell_\texttt{IB}^t = \lambda_t H_t\), where the modulation coefficient \(\lambda_t \in [-1, 1]\) depends on the value of \(H\left(o_t \mid o_{<t}, \bm{q}, \bm{a}\right)\), and tends to be smaller for more critical or informative tokens.

To obtain a practical estimate of \(\lambda_t\), we approximate it using the {\it negative} of token advantage \(A_t = A(o_t ; o_{<t}, \bm{q})\), which also captures token importance and is readily available in existing RL frameworks. This leads to a practical approximation of the IBRO objective:
$$
\min \; \mathcal{L}_\texttt{IB} = - \sum_{t=1}^T A_t\, H_t \quad \Longleftrightarrow \quad \max \; \mathcal{J}_\texttt{IB} = \sum_{t=1}^T A_t\, H_t,
$$
where maximizing \(A_t H_t\) encourages higher entropy for critical tokens, {\it i.e.}, tokens with large advantages, and penalizes less informative ones. The term \(\mathcal{J}_\texttt{IB}\) can be seamlessly incorporated into standard RL objectives such as PPO or GRPO as an additional \textit{IB regularization} term:
\begin{equation*}
\max \; \mathcal{J} = \mathcal{J}_\texttt{RL} + \alpha\, \mathcal{J}_\texttt{IB},
\end{equation*}
where \(\mathcal{J}_\texttt{RL}\) denotes the base RL objective, such as \(\mathcal{J}_\texttt{PPO}\), and \(\alpha > 0\) controls the regularization strength.

\paragraph{Efficiency of IB Regularization} Since both \(A_t\) and \(\pi_\theta(o_t \mid o_{<t}, \bm{q})\) have been obtained during RL objective computation, the proposed IB regularization is highly efficient and introduces negligible cost to the training pipeline. Moreover, the corresponding implementation is embarrassingly simple, requiring only a single line of code modification, as shown in \Cref{lst:policy_loss}. A complete modification based the widely-used LLM post-training framework VeRL \citep{sheng2025hybridflow} is provided in \Cref{app:code}.

\begin{figure}[t]

\centering
\begin{tcolorbox}[
    colframe=gray,       
    colback=white,       
    boxrule=0.5pt,       
    arc=2pt,             
    left=0pt, right=4pt, top=1pt, bottom=2pt, 
    width=0.98\linewidth, 
    enhanced
]
\noindent
\begin{normalcode}
def compute_pg_loss(log_prob, old_log_prob, advantage, clip_cfg):
\end{normalcode}
\begin{normalcode}
    pg_loss = compute_ppo_loss(log_prob, old_log_prob, advantage, clip_cfg)
\end{normalcode}
\begin{normalcode}
    entropy = compute_entropy(log_prob)
\end{normalcode}
\begin{diffdelcode}
-   entropy_loss = compute_mean(entropy)
\end{diffdelcode}
\begin{diffaddcode}
+   entropy_loss = compute_mean(entropy * advantage)
\end{diffaddcode}
\begin{normalcode}
    pg_loss = pg_loss - entropy_coeff * entropy_loss
\end{normalcode}
\begin{normalcode}
    return pg_loss
\end{normalcode}
\end{tcolorbox}
\captionof{listing}{Pseudocode for computing the policy loss with IB regularization. Both \texttt{entropy} and \texttt{advantage} are token-level tensors. Only a single line is modified to incorporate IB regularization.}
\label{lst:policy_loss}
\end{figure}

\section{Experiments}

In this section, we evaluate the effectiveness of our IB regularization approach across various mathematical reasoning benchmarks and RL algorithms.

\paragraph{Datasets and Models}  
We perform RL post-training on the DAPO-Math-17K dataset \citep{yu2025dapo}, which consists of $17{,}000$ mathematical questions, each paired with an integer answer. As the base model, we use Qwen2.5-7B \citep{yang2024qwen2_5}, a widely adopted pre-trained LLM that has not undergone any instruction tuning or reasoning-specific training. This makes it a suitable testbed for evaluating the effectiveness of algorithms aimed at incentivizing reasoning ability. During the post-training, we evaluate the mathematical reasoning performance of our models on three benchmark datasets: AMC23, AIME24, and AIME25. Specifically, AMC23 contains $40$ problems drawn from the 2023 American Mathematics Competitions. AIME24 and AIME25 each include $30$ problems from the 2024 and 2025 editions of the American Invitational Mathematics Examination, respectively. Compared to AMC23, the AIME datasets feature substantially more challenging problems and serve as stronger indicators of advanced reasoning performance.


\paragraph{Setup}  
We conduct LLM RL post-training based on VeRL framework \citep{sheng2025hybridflow}. To evaluate our approach, we adopt two representative RL algorithms: PPO \citep{schulman2017proximal}, which requires learning critic model, and DAPO \citep{yu2025dapo}, a GRPO variant that operates without a critic. The maximum response length is set to $20{,}480$ tokens, and no KL regularization is applied {\it w.r.t.} the reference policy during training. In our PPO setup, we incorporate the \texttt{ClipHigher} strategy from DAPO \citep{yu2025dapo} to mitigate entropy collapse and assess the compatibility of our IB regularization with this effective technique. The clipping parameters are set to $\texttt{clip\_low} = 0.2$ and $\texttt{clip\_high} = 0.28$. 
We train $2000$ steps in PPO and $7200$ steps in DAPO. Additional hyperparameter details are provided in \Cref{app:implementation details}. 
Performance is reported using the \texttt{avg@32} metric, which measures the average pass rate over $32$ sampled generations per prompt.


\paragraph{Baselines} (1) \texttt{No reg}: vanilla RL post-training without entropy regularization; (2) \texttt{Naive reg}: RL training with standard entropy regularization, using $\alpha=0.001$;
(3) \texttt{IB reg}: RL training with our proposed IB regularization, setting $\alpha=0.005$ instead of $0.001$, to maintain sufficient learning signals, given the presence of both positive and negative advantages in the objective calculation.


\begin{table}[t]
\centering
\caption{\texttt{avg@32} scores of different regularization strategies on AMC23, AIME24, and AIME25. Boldface indicates the highest score within each RL algorithm (PPO or DAPO). \texttt{Base} reports the model performance prior to post-training.}
\label{tab:results}
\resizebox{\textwidth}{!}{
\begin{tabular}{cc cccc  cccc}
\toprule
& \multirow{2}{*}{Method} & \multicolumn{2}{c}{AMC23} & \multicolumn{2}{c}{AIME24} & \multicolumn{2}{c}{AIME25} & \multicolumn{2}{c}{Avg} \\ 
\cmidrule(lr){3-4} \cmidrule(lr){5-6} \cmidrule(lr){7-8}  \cmidrule(lr){9-10}
& & \texttt{top@1} & \texttt{top@10} & \texttt{top@1} & \texttt{top@10} & \texttt{top@1} & \texttt{top@10} & \texttt{top@1}& \texttt{top@10} \\ 
\specialrule{1pt}{0.2\jot}{0.15pc}
& \texttt{Base} & \multicolumn{2}{c}{$17.3$} & \multicolumn{2}{c}{$1.5$} & \multicolumn{2}{c}{$1.2$} & \multicolumn{2}{c}{$6.7$}  \\
\specialrule{0.01pt}{0.2\jot}{0.15pc}
\cellcolor{blue!15}  & \texttt{No reg} & $63.8$ & $62.8 $& $17.7$ & $16.8$ & $13.1$ & $11.9$ & $31.5$ & $30.5$ \\ 
\cellcolor{blue!15}  & \texttt{Naive reg} & $63.3$ & $62.3$ & $15.0$ & $14.3$ & $10.3$ & $9.5$ & $29.5$ & $28.7$ \\ 
\multirow{-2.8}{*}{\cellcolor{blue!15}\rotatebox{90}{PPO}} & \texttt{IB reg} & \bm{$67.3$} & \bm{$66.8$} & \bm{$20.3$} & \bm{$19.8$} & \bm{$13.6$} & \bm{$13.0$} & \bm{$33.7$} & \bm{$33.2$} \\ 
\midrule
\cellcolor{blue!15}  & \texttt{No reg} & \bm{$86.3$} & \bm{$85.7$} & $18.6$ & $18.2$ & $17.0$ & \bm{$16.3$} & $40.6$ & $40.1$ \\ 
\cellcolor{blue!15}  & \texttt{Naive reg} & $82.5$ & $82.3$ & $20.3$ & $19.7$ & $11.6$ & $11.1$ & $38.1$ & $37.7$ \\ 
\multirow{-2.8}{*}{\cellcolor{blue!15}\rotatebox{90}{DAPO}} & \texttt{IB reg} & {$85.1$} & {$84.3$} & \bm{$25.4$} & \bm{$24.6$} & \bm{$17.7$} & \bm{$16.3$} & \bm{$42.7$} & \bm{$41.7$} \\ 
\bottomrule
\end{tabular}}
\end{table}

\begin{figure}[t]
\centering
\subfigure[PPO]{
\begin{minipage}[b]{0.48\textwidth}
\centering
    		\includegraphics[width=\columnwidth]{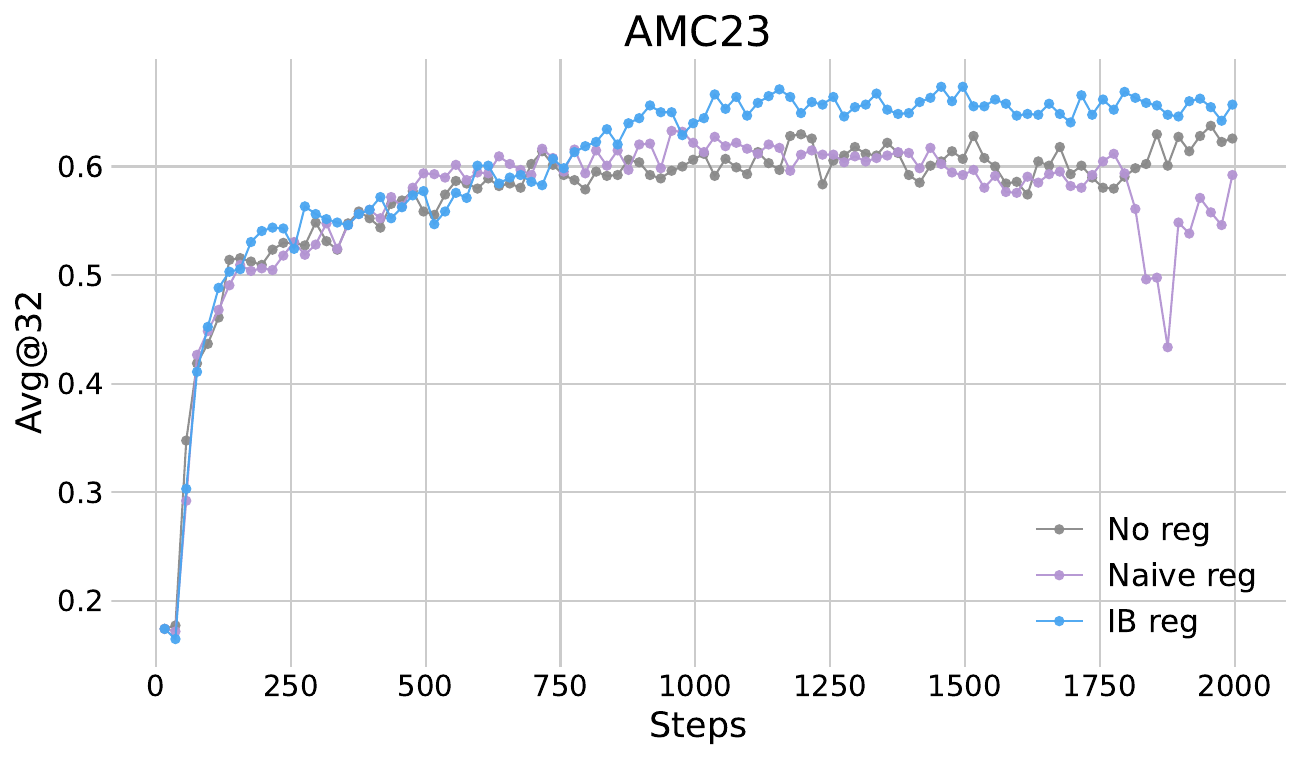}
                \includegraphics[width=\columnwidth]{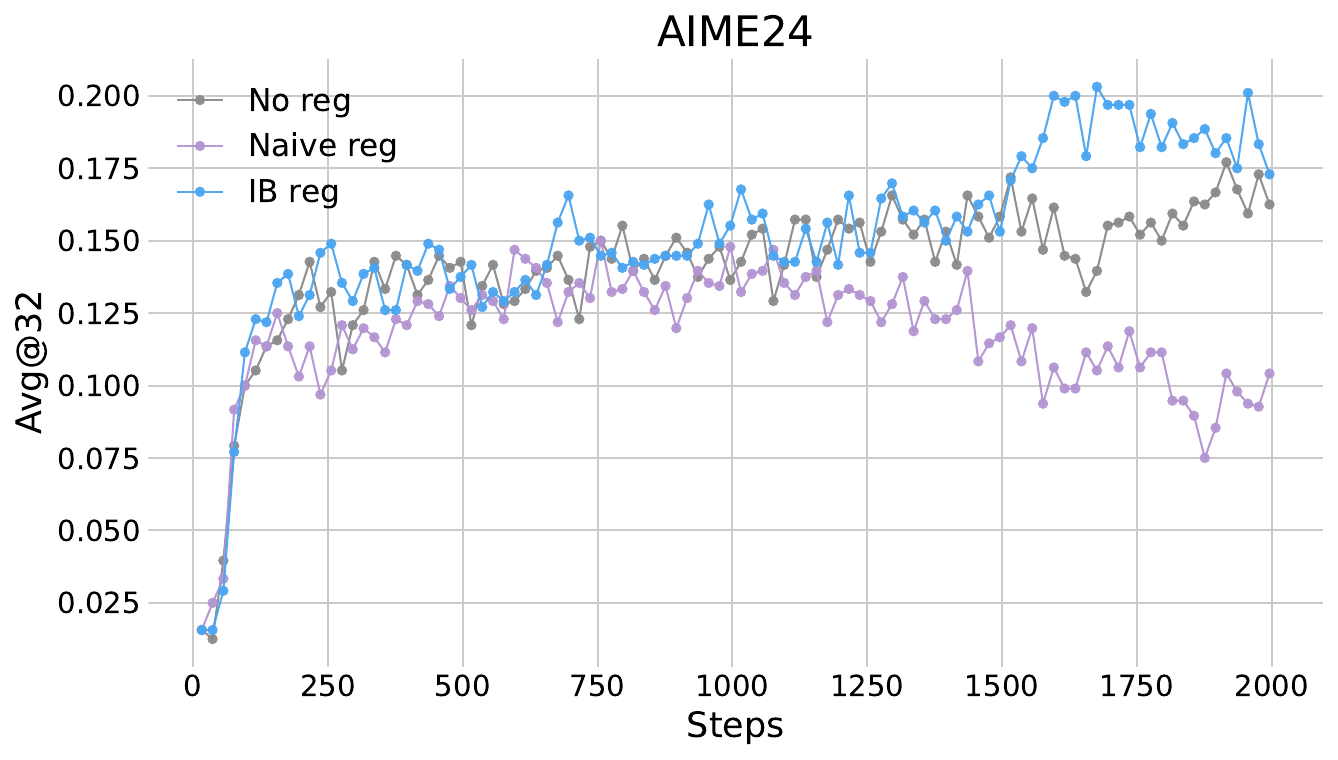}
                \includegraphics[width=\columnwidth]{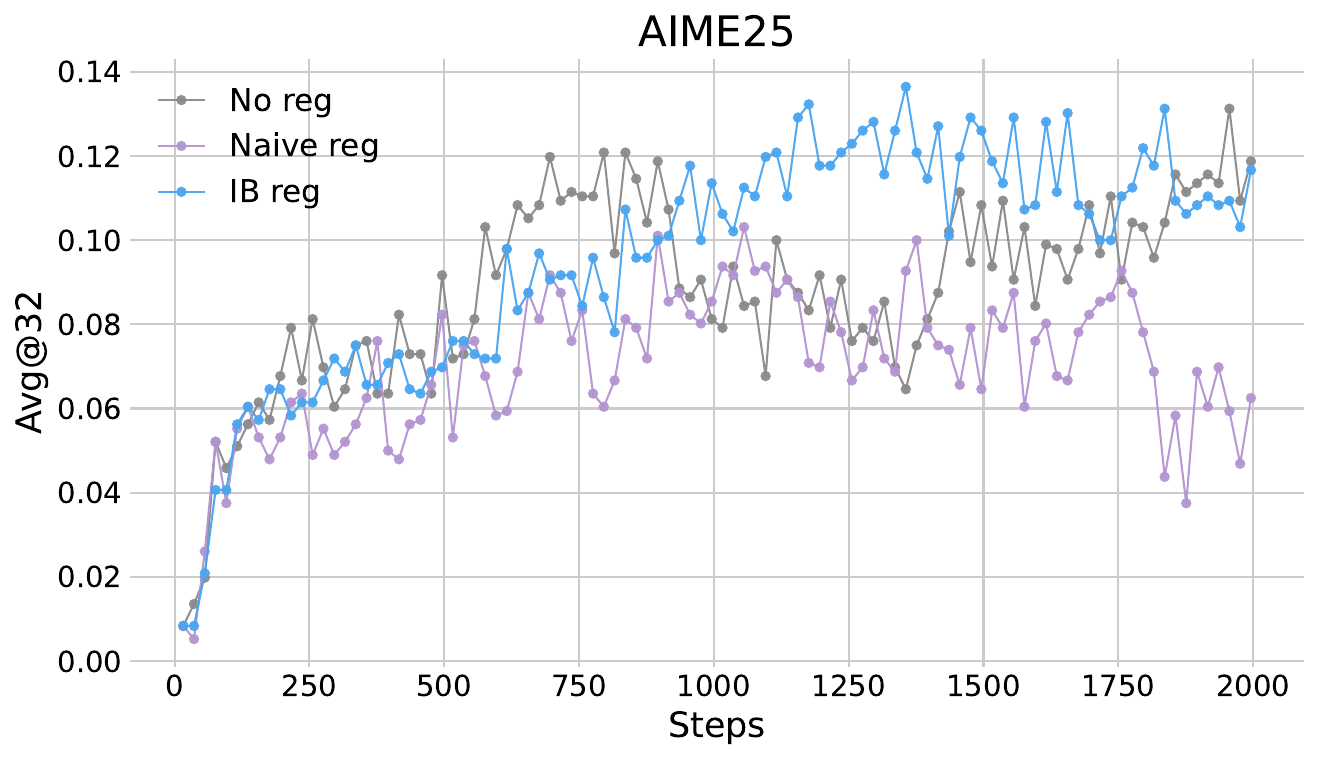}
    		\end{minipage}
		\label{figure:ppo}  
    	}
\subfigure[DAPO]{
\begin{minipage}[b]{0.48\textwidth}
\centering
    		\includegraphics[width=\columnwidth]{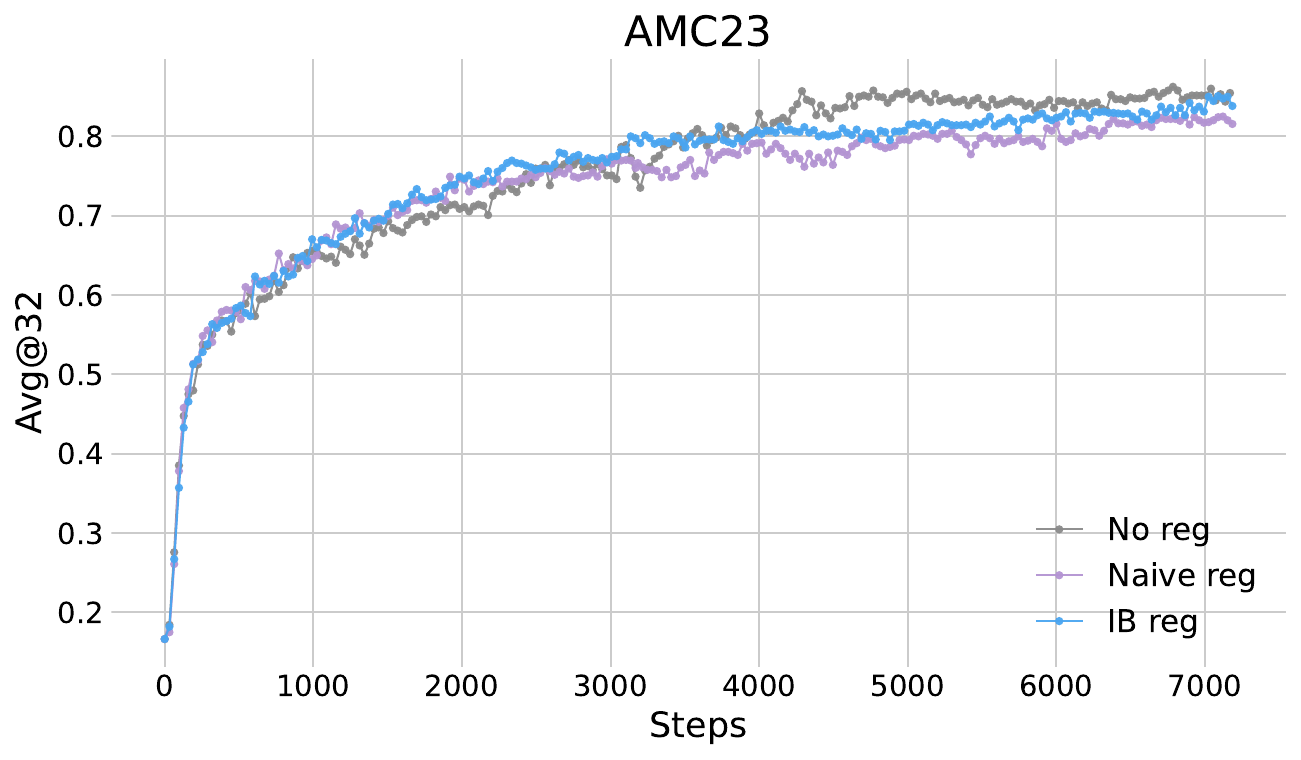}
                \includegraphics[width=\columnwidth]{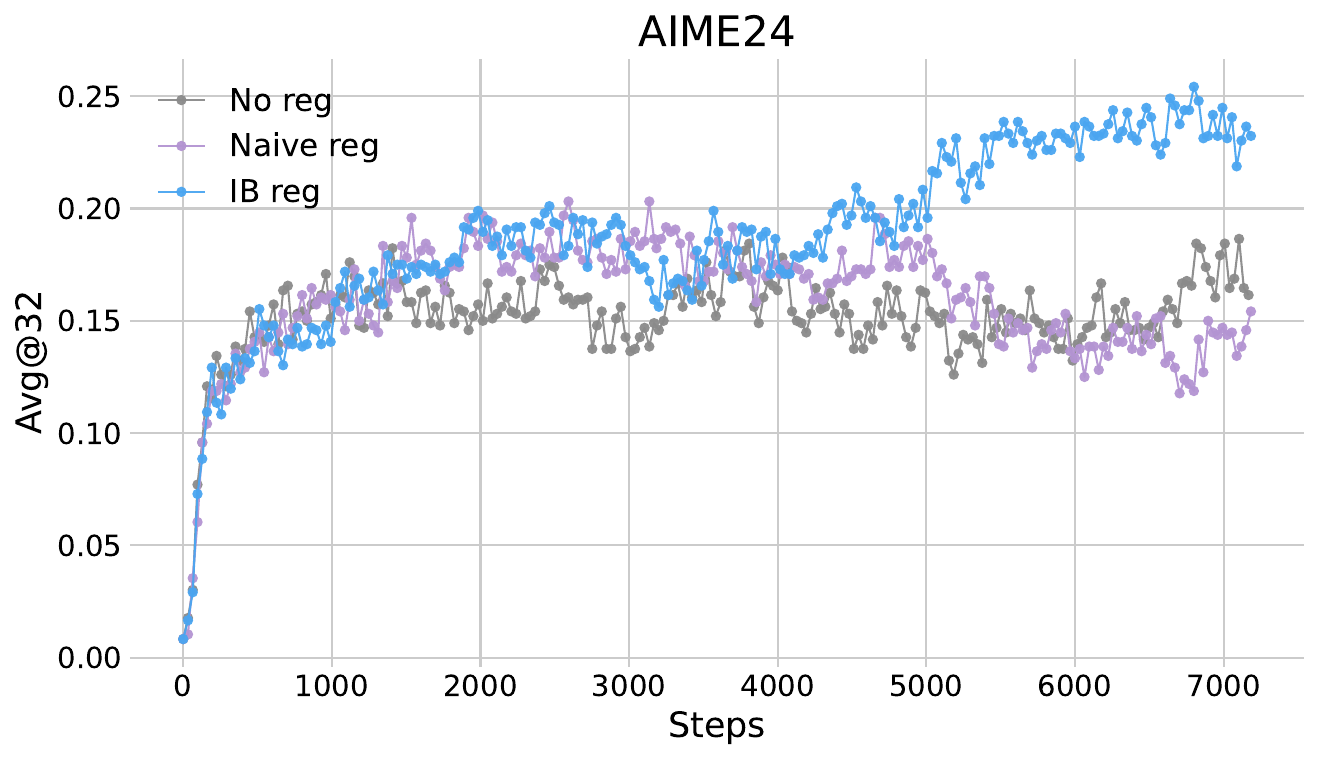}
                \includegraphics[width=\columnwidth]{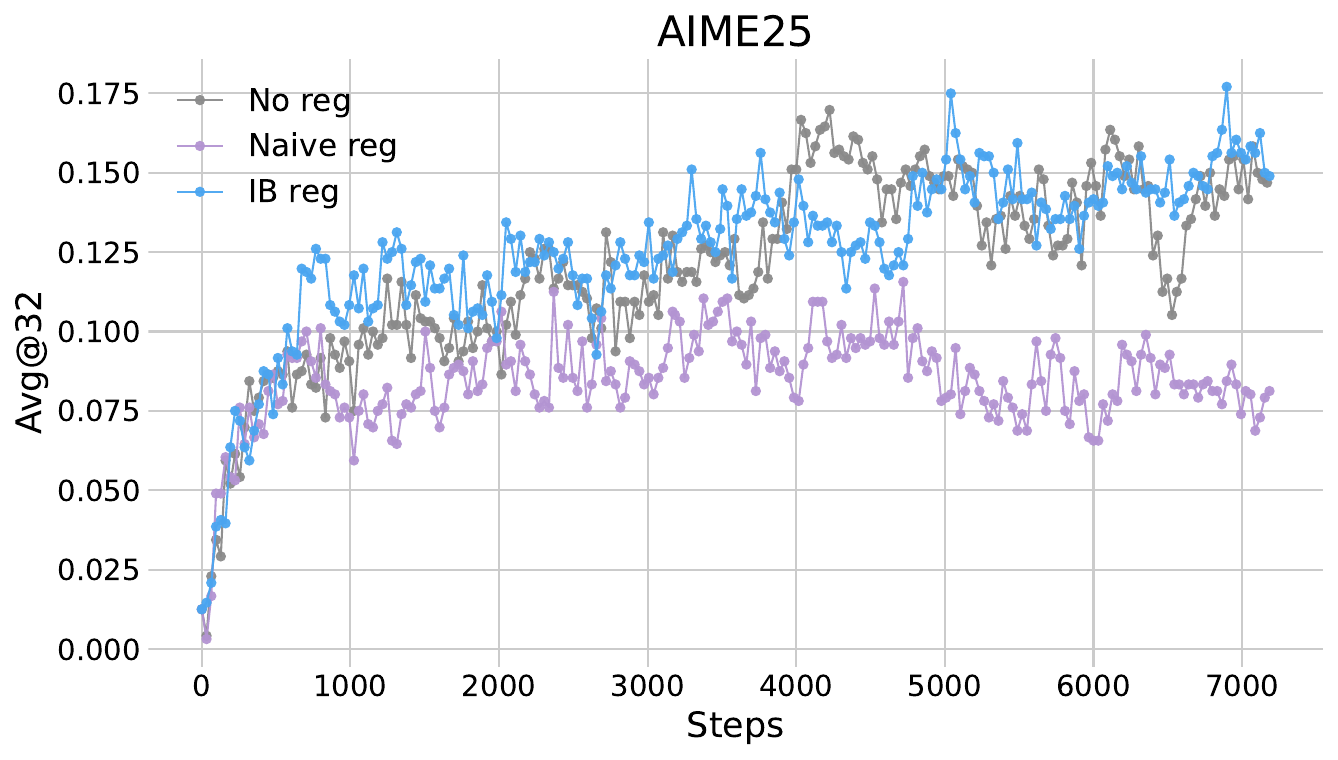}
    		\end{minipage}
		\label{figure:dapo}  
    	}
\caption{Plots of \texttt{avg@32} as functions of training steps in (a) PPO and (b) DAPO.}
\label{figure:curve}
\end{figure}

\subsection{Main Results}

The evaluation results are presented in \Cref{tab:results}. To robustly assess performance improvements, we report both the best score using \texttt{top@1} and the average score across the top ten checkpoints using \texttt{top@10}. Our empirical findings reveal several key insights. First, naive entropy regularization (\texttt{Naive reg}) consistently underperforms compared to vanilla training (\texttt{No reg}), with scores dropping from $31.5$ to $29.5$ in PPO and from $40.6$ to $38.1$ in DAPO. This suggests that indiscriminate entropy injection can degrade reasoning performance. Second, our proposed IB regularization (\texttt{IB reg}) yields consistent and substantial improvements over the baseline, with average gains of two points in both PPO ($31.5 \rightarrow 33.7$) and DAPO ($40.6 \rightarrow42.7$). These results demonstrate the effectiveness and robustness of our method.

To further support these findings, we present the training curves in \Cref{figure:curve}. As illustrated for PPO in \Cref{figure:ppo}, the \texttt{IB reg} curves consistently outperform the baselines across all evaluation benchmarks. For DAPO, shown in \Cref{figure:dapo}, although the \texttt{No reg} baseline performs best on AMC23, our \texttt{IB reg} achieves superior results on the remaining benchmarks, most notably on AIME24, where it reaches a score of $25.4$, while the baselines plateau around $20$.

\paragraph{Entropy Dynamics} 

Beyond performance metrics, we examine the entropy dynamics during post-training under both PPO and DAPO, as illustrated in \Cref{figure:entropy}. Compared to the no-regularization baseline, naive entropy regularization partially mitigates entropy collapse. However, it frequently drives the entropy to excessively high levels in the later stages of training, sometimes even exceeding its initial value. In contrast, IB regularization maintains entropy at a similar magnitude to the vanilla baseline, while exhibiting more stable and controlled behavior throughout training.

These observations offer several insights. First, given the inferior empirical performance of naive entropy regularization, excessive entropy can be as detrimental as entropy collapse for LLM reasoning, leading to unfocused exploration. Second, IB regularization does not seek to increase token entropy uniformly. Instead, it selectively redistributes entropy by encouraging higher entropy for critical tokens that benefit from exploration, while reducing entropy for less informative tokens to maintain coherence and fluency. Crucially, as IB regularization preserves the overall entropy scale, it is highly compatible with existing training pipelines: in practice, modulating the degree of off-policyness in RL, tuning sampling temperature, or employing the \texttt{ClipHigher} strategy, provides a more stable and general mechanism for maintaining a balanced entropy range, compared to explicit entropy regularization \citep{he2025skywork,yu2025dapo,liu2025scalingrlunlockingdiverse}. As a result, our IB regularization can be seamlessly integrated into well-tuned setups without significantly disrupting the entropy dynamics.

\paragraph{Response Length Analysis} 
Previous studies have shown that improvements in response length are often closely associated with gains in reasoning accuracy, as longer responses tend to reflect more complete and detailed reasoning processes. \Cref{figure:resp_len} illustrates the evolution of mean response length throughout post-training. We observe two key patterns: (1) IB regularization does not consistently produce longer responses compared to the vanilla baseline. Concretely, it yields shorter responses under PPO and longer ones under DAPO. Nevertheless, the response length under IB regularization exhibits stable growth and consistently remains within a desirable range of $2$K$-$$3$K tokens; and (2) naive entropy regularization tends to shorten responses, especially under PPO, with a less pronounced effect in DAPO. This outcome, while somewhat counter-intuitive, has not been formally documented in prior work. It challenges the common belief that higher entropy promotes greater exploration and thus results in longer reasoning trajectories.

We attribute this phenomenon to the behavior of the end-of-sequence token (\texttt{[EOS]}). During early decoding stages, the probability of emitting \texttt{[EOS]} is typically low. However, naive entropy regularization increases entropy uniformly across all tokens, which flattens the output probability distribution and can inadvertently raise the likelihood of generating \texttt{[EOS]}, thereby causing premature truncation and shorter responses. In contrast, IB regularization selectively increases entropy for critical tokens while suppressing it for less informative ones. This suppression often raises the relative probability of sampled uninformative tokens while reducing that of others, such as \texttt{[EOS]}. As a result, IB regularization mitigates premature termination and better preserves the response length required for effective multi-step reasoning.

\begin{figure}[t]
\centering
\includegraphics[width=0.48\columnwidth]{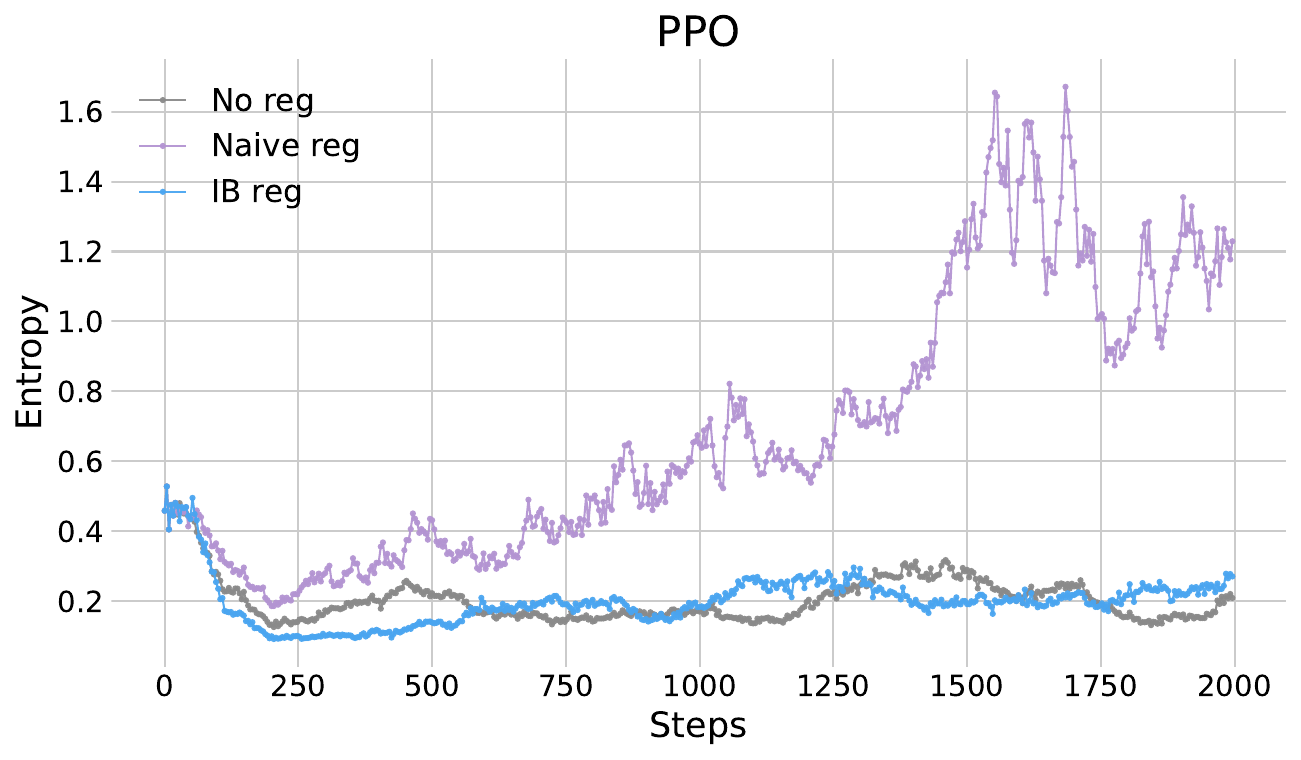}
\includegraphics[width=0.48\columnwidth]{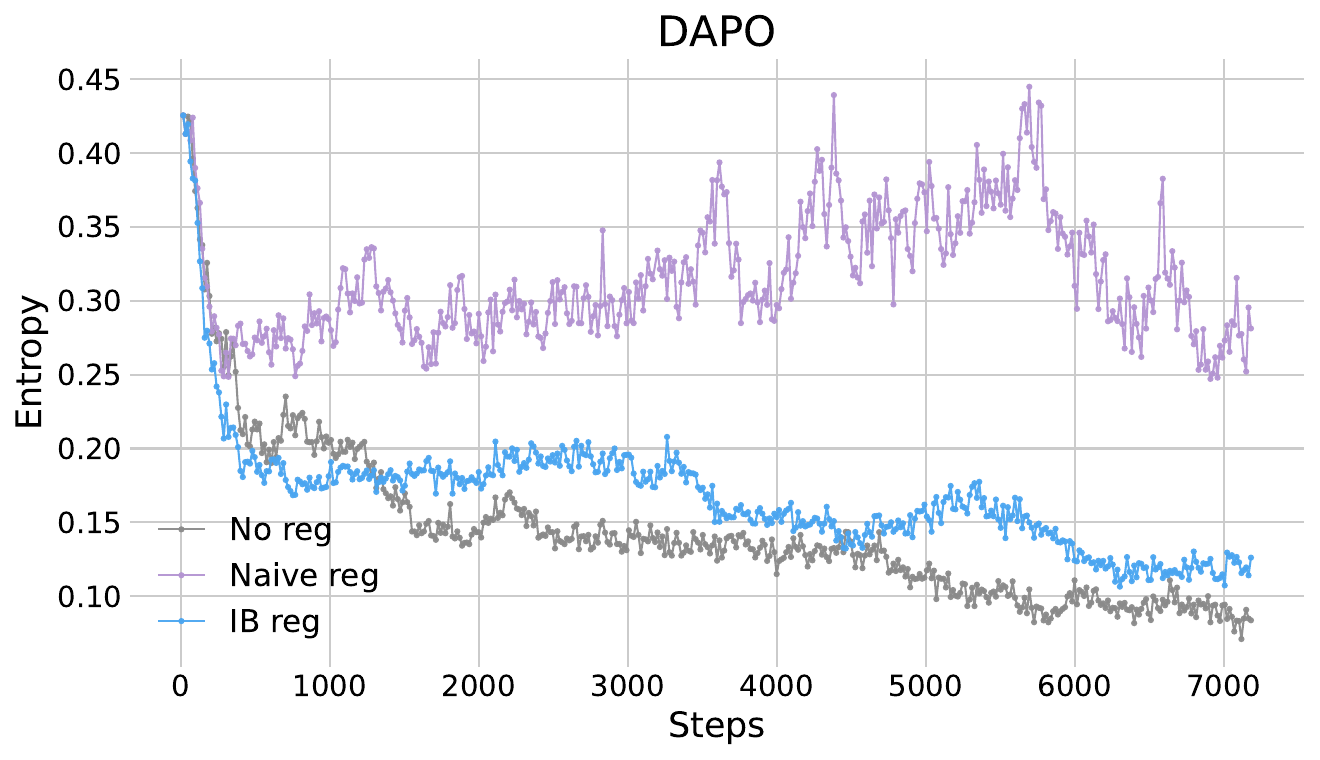}
\caption{Plots of entropy as functions of training steps.}
\label{figure:entropy}
\end{figure}

\begin{figure}[t]
\centering
\includegraphics[width=0.48\columnwidth]{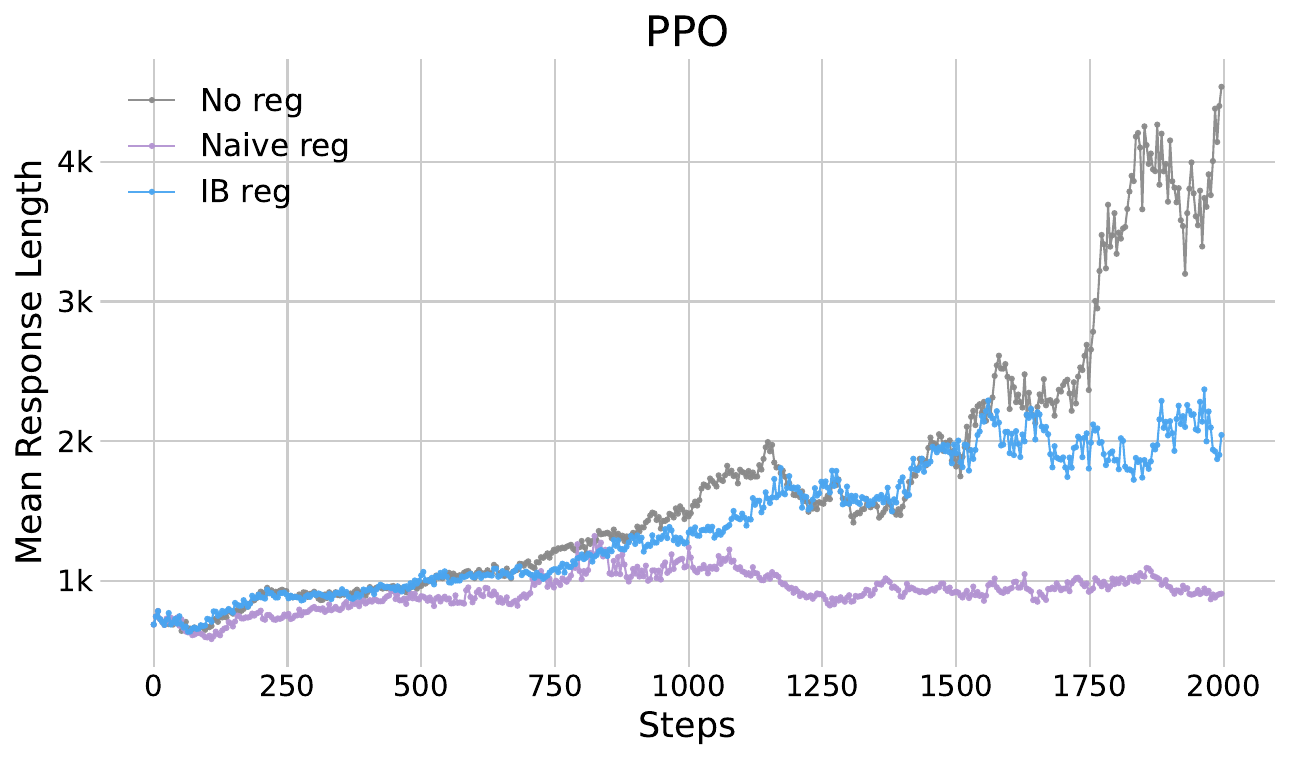}
\includegraphics[width=0.48\columnwidth]{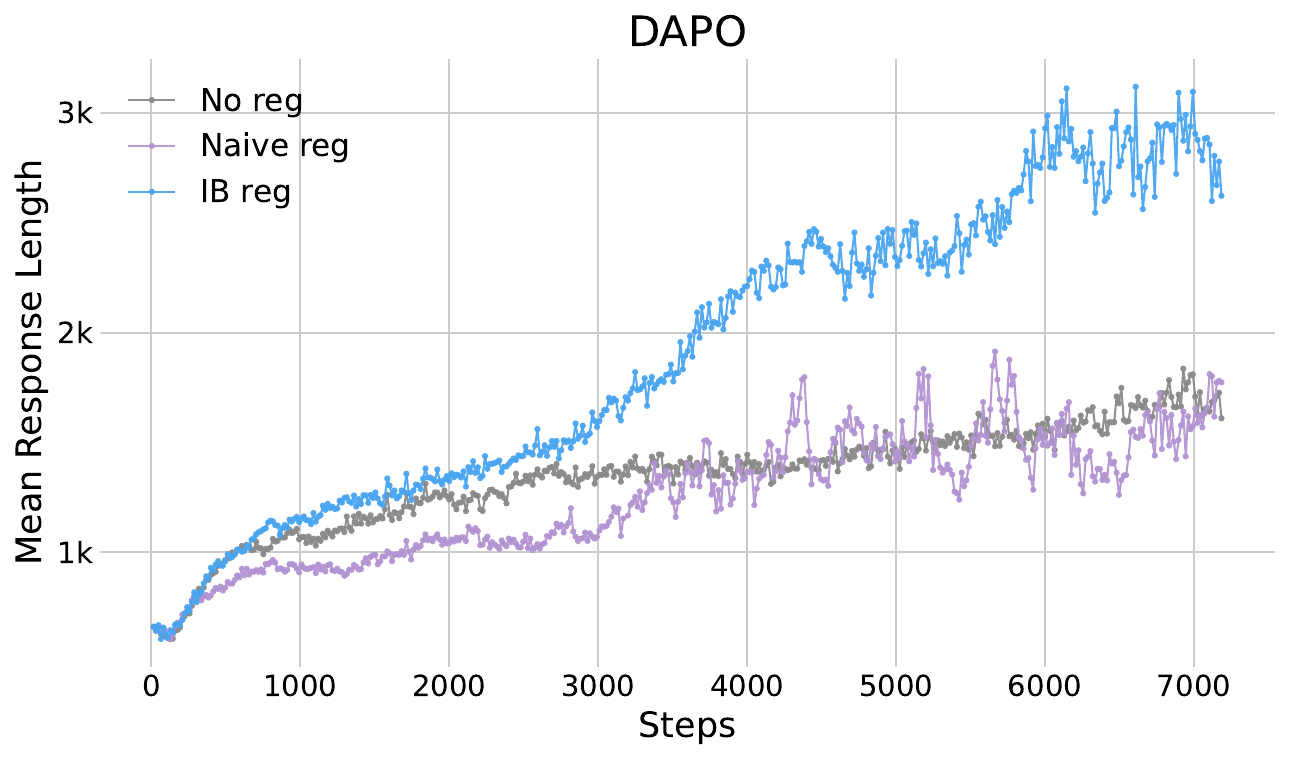}
\caption{Plots of mean response length as functions of training steps.}
\label{figure:resp_len}
\end{figure}


\section{Discussion and Limitation}

\paragraph{Discussion}
The proposed IB regularization can be interpreted as a token-level entropy regularization scheme, where the regularization strength for each token is weighted by its corresponding advantage. In PPO, token-level advantages vary within a single response and are provided by the critic model. As a result, critical tokens, {\it i.e.}, tokens with higher advantages, receive stronger entropy regularization. This mechanism focuses optimization on informative positions, promotes targeted exploration, and contributes to improved reasoning generalization.

In contrast, GRPO and DAPO assign the same scalar advantage to all tokens within a response, based on the group-normalized final reward. Under this setting, IB regularization degrades to applying a positive entropy regularization to correct responses and a negative one to incorrect responses. Since updates on correct responses tend to concentrate the output distribution and reduce entropy, while updates on incorrect responses often flatten the distribution and increase entropy, IB regularization serves as a soft constraint or entropy-aware damping force that counteracts excessive entropy reduction on correct responses and curbs entropy explosion on incorrect ones. This mechanism encourages conservative token-level entropy adjustments, helping to maintain a well-balanced entropy profile and ensuring stable training.

Moreover, in \Cref{sec:ib reg}, we derive IB regularization by setting $\beta = 2$, yielding a symmetric range of $[-H_t, H_t]$. However, our IBRO framework is more flexible and allows for an additional hyperparameter $\eta \in \mathbb{R}$, which leads to an generalized formulation $\mathcal{J}_\texttt{IB} = (A_t + \eta) H_t$. This extension enables asymmetric control over the entropy shaping and offers finer-grained trade-off between exploration and exploitation.

\paragraph{Limitation} 
While our results demonstrate the effectiveness of IB regularization, several limitations remain. First, the regularization strength coefficient requires careful tuning to achieve optimal performance, and its optimal value may vary across tasks, model sizes, and training stages. Developing automated or adaptive tuning strategies remains an open challenge. Second, due to computational constraints, our experiments are conducted on models with $7$B parameters. The scalability and effectiveness of IB regularization on larger LLMs, such as those with $32$B parameters or more, have not yet been validated. Exploring this setting is non-trivial, as running RL post-training on very large LLMs demands substantially greater computational resources, potentially exceeding an order of magnitude beyond our current setup. We leave this investigation to future work.

\section{Conclusion}

In this paper, we introduce an information-theoretic framework called \textit{\textbf{i}nformation \textbf{b}ottleneck-aware \textbf{r}easoning \textbf{o}ptimization} (IBRO) to optimize reasoning trajectories in large language models (LLMs). Grounded in information bottleneck principle, IBRO formalizes reasoning effectiveness by encouraging trajectories to be informative regarding the correct answer while remaining generalizable across different prompts. We further derived a token-level surrogate objective and proposed a practical approximation, termed \textit{IB regularization}, which modulates token-level entropy according to token-level advantage readily available from RL frameworks. The method requires no additional computation and can be implemented with only one line of code modification. Empirical evaluations across several mathematical reasoning benchmarks demonstrate that integrating IB regularization into existing RL algorithms consistently enhances reasoning accuracy and stability. Our results underline the significance of information-theoretic insights in optimizing LLM reasoning, providing theoretical foundations and practical tools for future research in this direction.

\bibliographystyle{unsrtnat}   

\bibliography{ref}
\newpage
\appendix

\section{Proof of Theorem 2}

\label{app:proof_main}
Our proof is based on the theoretical results in \citep{kawaguchi2023does}, which we first outline a simply version as below.

\paragraph{Notation} We denote the input and output variables by \( X \) and \( Y \), respectively. Consider a neural network \( f = g \circ \phi \), which comprises two components: an encoder \( \phi \) that maps inputs \( X \) to latent features \( Z = \phi(X) \), and a predictor \( g \) that generates predictions \( g(Z) \) based on these latent representations. Let \(\mathcal{S} = \{(\bm{x}_i, \bm{y}_i)\}_{i=1}^m\) be a training dataset consisting of \( m \) examples drawn independently and identically distributed (i.i.d.) from a joint distribution \(\mathcal{P}\) over \(\mathcal{X} \times \mathcal{Y}\), with \(\bm{x}_i \in \mathcal{X}\) and \(\bm{y}_i \in \mathcal{Y}\). Given a bounded loss function \(\ell: \mathcal{X}\times\mathcal{Y}\rightarrow\mathbb{R}^+\), the generalization gap, defined as the difference between the expected and empirical losses, is expressed as:
\[
\Delta(\mathcal{S}) := \mathbb{E}_{(X, Y)\sim \mathcal{P}}\left[\ell(f^\mathcal{S}(X), Y)\right] - \frac{1}{m}\sum_{i=1}^m\ell(f^\mathcal{S}(\bm{x}_i), \bm{y}_i),
\]
where \( f^\mathcal{S} \) denotes the neural network trained on the dataset \(\mathcal{S}\).

\begin{theorem}[Theorem $2$ in \citep{kawaguchi2023does}] 
\label{thm:kawa}
Given a network $f=g\circ \phi$ trained on the dataset $\mathcal{S}$, the dataset size $|\mathcal{S}|=m$. Then, for any $\delta >0$, with probability at least $1-\delta$ over the training set $\mathcal{S}$, the following generalization bound holds:
\begin{equation}
   \Delta(\mathcal{S})  \lesssim \sqrt{\frac{I(X;Z | Y) + I(\phi;{S})+ H(Z|X,Y)+\log\frac{1}{\delta}}{m}} + \tilde{\mathcal{O}}\left(\sqrt{\frac{I(\phi,S)+1}{m}}\right)
\end{equation}
\end{theorem}

To upper bound the mutual information between model parameters \( \phi \) and dataset \( {S} \), we invoke the following lemma.
\begin{lemma}
\label{lemma:kl-bound}
Let the encoder parameters \( \theta \in \mathbb{R}^d \) be updated from initialization \( \theta_0 \) to \( \theta = \theta_0 + \Delta\theta \) after training on dataset \( {S} \).  
Let the prior distribution be \( P_0 = \mathcal{N}(\theta_0, \sigma^2 I) \), and the posterior \( P_{\Theta|{S}} = \delta(\theta_0 + \Delta\theta) \).  
Then, for any \( \sigma > 0 \), the mutual information satisfies:
\[
I(\theta; {S}) \leq \frac{\|\Delta \theta\|^2}{2 \sigma^2} + \frac{d}{2} \log(2\pi \sigma^2).
\]
\end{lemma}

\begin{proof}
By definition,
\[
I(\theta; {S}) = \mathbb{E}_{{S}} \left[ \mathrm{KL}(P_{\Theta|{S}} \| P_{\Theta}) \right] \leq \mathbb{E}_{{S}} \left[ \mathrm{KL}(P_{\Theta|{S}} \| P_0) \right],
\]
using convexity of KL and data-independent prior \( P_0 \).  
Since \( P_{\Theta|{S}} = \delta(\theta_0 + \Delta\theta) \) and \( P_0 = \mathcal{N}(\theta_0, \sigma^2 I) \), we compute
\[
\mathrm{KL}(\delta(\theta_0 + \Delta\theta) \,\|\, \mathcal{N}(\theta_0, \sigma^2 I)) = \frac{\|\Delta\theta\|^2}{2\sigma^2} + \frac{d}{2} \log(2\pi \sigma^2).
\]
\end{proof}
As $H(\bm{r}\mid \bm{q},\bm{a})=H(\bm{r}\mid \bm{q})$ under the Markov chain $\bm{a} \leftrightarrow \bm{q} \leftrightarrow \bm{r}$, we have
\begin{align*}
I(\bm{q}; \bm{r} \mid \bm{a}) = H(\bm{r}\mid \bm{a}) - H(\bm{r}\mid \bm{q},\bm{a}) = H(\bm{r} \mid \bm{a}) - H(\bm{r} \mid \bm{q}).
\end{align*}
Then
\[
H(\bm{r} \mid \bm{a}) \leq  H(\bm{r},\bm{q} \mid \bm{a}) = H(\bm{r} \mid \bm{q}, \bm{a}) + H(\bm{q} \mid \bm{a}),
\]
which gives the inequality
\[
I(\bm{q}; \bm{r} \mid \bm{a}) \leq H(\bm{r} \mid \bm{q}, \bm{a}) + H(\bm{q} \mid \bm{a}) - H(\bm{r} \mid \bm{q}).
\]
Recall the surrogate IBRO loss
\[
\mathcal{L}_\texttt{IB} = \beta\, H(\bm{r} \mid \bm{q}, \bm{a}) - H(\bm{r} \mid \bm{q}),
\]
then for \( \beta \geq 2 \), we obtain
\[
I(\bm{q}; \bm{r} \mid \bm{a}) + H(\bm{r} \mid \bm{q}, \bm{a}) \leq 2 H(\bm{r} \mid \bm{q}, \bm{a}) + H(\bm{q} \mid \bm{a}) - H(\bm{r} \mid \bm{q}) \leq \mathcal{L}_\texttt{IB} + H(\bm{q} \mid \bm{a}).
\]
Therefore, under the mild condition \( \beta \geq 2 \), the quantity \( \mathcal{L}_\texttt{IB} + H(\bm{q} \mid \bm{a}) \) serves as an upper bound on \( I(\bm{q}; \bm{r} \mid \bm{a}) + H(\bm{r} \mid \bm{q}, \bm{a}) \). Since \( H(\bm{q} \mid \bm{a}) \) is a constant that depends only on the data distribution \( \mathcal{P} \), combining \Cref{thm:kawa} and \Cref{lemma:kl-bound} yields the desired bound.

\newpage

\section{IB Regularization Modification on VeRL}
\label{app:code}
We provide an example of modifying the VeRL framework to switch from naive entropy regularization to our proposed IB regularization; please refer to \Cref{lst:verl modification}. This change requires only three lines of code in \href{https://github.com/volcengine/verl/blob/916ab431b7956480ba75cbbb323b1979ce8e2743/verl/workers/actor/dp_actor.py#L403}{verl/workers/actor/dp\_actor.py} (with PyTorch FSDP Backend).


\begin{figure}[ht]

\centering
\begin{tcolorbox}[
    colframe=gray,       
    colback=white,       
    boxrule=0.5pt,       
    arc=2pt,             
    left=0pt, right=4pt, top=1pt, bottom=2pt, 
    width=0.98\linewidth, 
    enhanced
]
\noindent
\begin{normalcode}
    if entropy_coeff != 0:\end{normalcode}
\begin{diffaddcode}
+       token_advantages = data["advantages"]    
\end{diffaddcode}
\begin{diffdelcode}
-       entropy_loss = agg_loss(loss_mat=entropy, loss_mask=response_mask, loss_agg_mode=loss_agg_mode)\end{diffdelcode}
\begin{diffaddcode}
+       entropy_loss = agg_loss(loss_mat=entropy * token_advantages, loss_mask=response_mask, loss_agg_mode=loss_agg_mode)
\end{diffaddcode}
\begin{normalcode}
        # compute policy loss
\end{normalcode}
\begin{normalcode}
        policy_loss = pg_loss - entropy_loss * entropy_coeff
\end{normalcode}
\begin{normalcode}
    else:
\end{normalcode}
\begin{normalcode}
        policy_loss = pg_loss
\end{normalcode}
\end{tcolorbox}
\vspace{-2mm}
\captionof{listing}{IB regulareization modification on VeRL.}
\label{lst:verl modification}
\end{figure}


\section{Implementation Details}
\label{app:implementation details}
The pre-trained LLM of Qwen2.5-7B can be download via \url{https://huggingface.co/Qwen/Qwen2.5-7B}. The training dataset DAPO-Math-17K is available at \url{https://huggingface.co/datasets/BytedTsinghua-SIA/DAPO-Math-17k}, and the evaluation datasets of AMC23, AIME24, and AIME25 can be download on \url{https://huggingface.co/math-ai}. \Cref{tab:hyperparams} lists the key hyperparameters used in PPO and DAPO. All experiments are conducted on 4$\times$8 NVIDIA H20 GPUs.


\begin{table}[h]
\centering
\small
\setlength{\tabcolsep}{6pt}
\renewcommand{\arraystretch}{1.1}
\caption{Key hyperparameters for PPO and DAPO. “—” denotes not used.}
\label{tab:hyperparams}
\begin{tabular}{lll}
\toprule
\textbf{Category} & \textbf{PPO} & \textbf{DAPO} \\
\midrule
\multicolumn{3}{l}{\textit{Sampling and Validation}} \\
\quad Temperature & $1.0$ & $1.0$ \\
\quad Top-$p$ / Val Top-$p$ & $1.0$ / $0.7$ & $1.0$ / $0.7$ \\
\midrule
\multicolumn{3}{l}{\textit{Clipping}} \\
\quad Clip ratio (low / high) & $0.2$ / $0.28$ & $0.2$ / $0.28$ \\
\midrule
\multicolumn{3}{l}{\textit{Sequence Limits}} \\
\quad Max prompt / response length & $2048$ / $20480$ & $2048$ / $20480$ \\
\quad Overlong buffer (on/off) & off & on \\
\quad Buffer length / penalty & — & 4096 / 1.0 \\
\midrule
\multicolumn{3}{l}{\textit{Batching}} \\
\quad Train batch / mini-batch size & $1024$ / $256$ & $512$ / $32$ \\
\quad Gen batch size & — & $512\times 3$ \\
\quad Responses per prompt & — & $16$ \\
\midrule
\multicolumn{3}{l}{\textit{Optimization}} \\
\quad Loss aggregation & \texttt{token-mean} & \texttt{token-mean} \\
\quad Actor LR & $1e-6$ & $1e-6$ \\
\quad Critic LR & $1e-5$ & — \\
\quad LR warmup steps & $10$ & $10$ \\
\quad Critic warmup steps & 5 & — \\
\bottomrule
\end{tabular}
\end{table}

\end{document}